\documentclass[reqno,12pt]{amsart}

\usepackage{amsaddr}
\usepackage{url}
\usepackage{bm}
\usepackage{amssymb}
\usepackage{minibox}
\usepackage{tikz}
\usetikzlibrary{calc}
\usepackage[symbol]{footmisc}

\newcommand\bi{\bm i}
\newcommand\bj{\bm j}
\newcommand\bk{\bm k}

\newcommand\starop[1]{\mathsf S_{#1}}

\newcommand\setinvertibledualquat{\mathbb D\mathbb H}
\newcommand\setdualquat{\mathfrak d \mathfrak h}
\newcommand\setunitdualquat{\mathbb S\mathbb D\mathbb H}
\newcommand\setvectordualquat{\mathfrak{s}\mathfrak d \mathfrak h}

\DeclareMathOperator\realpart{Re}
\DeclareMathOperator\imagpart{Im}
\DeclareMathOperator\size{{size}}

\newcommand{\liederiv}{\mathcal L}

\newtheorem{Lemma}{Lemma}
\newtheorem{Theorem}{Theorem}

\begin{document}

\title[Lie derivatives with dual quaternions]{Using Lie derivatives with dual quaternions for parallel robots}

\author{Stephen Montgomery-Smith}
\address{Department of Mathematics, University of Missouri, Columbia, MO 65211.\\
\rm\url{stephen@missouri.edu}\\
\rm\url{https://stephenmontgomerysmith.github.io}}

\author{Cecil Shy}
\address{Johnson Space Center, 2101 E.~NASA Parkway, Houston, TX 77058.\\
\rm\url{cecil.shy-1@nasa.gov}}

\keywords{forward kinematics, dynamics, pose, twist, wrench, normalization, Stewart Platform}

\date{November 27, 2023}

\begin{abstract}
We introduce the notion of the Lie derivative in the context of dual quaternions that represent rigid motions and twists.  First we define the wrench in terms of dual quaternions.  Then we show how the Lie derivative helps understand how actuators affect an end effector in parallel robots, and make it explicit in the two cases case of Stewart Platforms, and cable-driven parallel robots.  We also show how to use Lie derivatives with the Newton-Raphson Method to solve the forward kinematic problem for over constrained parallel actuators.  Finally, we derive the equations of motion of the end effector in dual quaternion form, which include the effect of inertia from the actuators.
\end{abstract}

\maketitle


\section{Introduction}

This paper is broadly about \emph{poses} and/or \emph{rigid motions}, and how they can be represented by \emph{unit dual quaternions}.  A pose is a description of a frame of reference with respect to a fixed frame of reference, and consists of an \emph{orientation}, and a \emph{position}.  A rigid motion consists of a \emph{rotation} followed by a \emph{translation}.  From a mathematical point of view, these can be considered to be the same thing, and so we use these terms interchangeably (but see \cite{chirikjian-et-al} for a different point of view).

Along with poses and/or rigid motions, we have the notion of \emph{screws}.  First we have rates of change of poses/rigid bodies, which are angular and translational velocities, and are described by \emph{twists}.  Second we have descriptions of how to change the inertia of the rigid body, the \emph{wrench}, that is, the torque and the force applied to the body.  For more reading on rigid body kinematics and dynamics, including screw theory, we refer the reader to \cite{ball,bottema-et-al,gallardo-alvarado,selig-book}.

The notion of the dual quaternion, and its use to represent poses and rigid motions, seems to go back to McAulay \cite{mcauley}, inspired by the earlier work by Clifford \cite{clifford}.  The notion of using dual quaternions to represent twists may be found in \cite{adorno,perez-et-al}.  A basic introduction to dual quaternions may be found in \cite{kenwright,montgomery-smith-et-al}, the latter also covering twists.  Many authors have used dual quaternions to represent hierarchies of poses, that is, chains of manipulators \cite{kenwright,perez-et-al,schilling1,schilling2,silva-et-al}.  Papers on representing kinematics or dynamics via dual quaternions include \cite{adorno,agrawal,dooley-et-al,han-et-al,kussaba-et-al,spong-et-al,wang-et-al}.  Dual quaternions have also found great use in computer graphics \cite{kavan-et-al,kavan-et-al-2}.

(The reader should be aware that \cite{adorno,han-et-al} have incorrect formulas for the logarithm and exponential of dual quaternions --- the correct formulas may be found in \cite{montgomery-smith-functional-calculus}, and \cite{selig} for the exponential.)


The purpose of this paper is to introduce the notion of using Lie derivatives for dual quaternions.  We show that these can be used to essentially automate the creation of rather complex formulas, which are required for forward kinematics, and for dynamic equations of motion.

The authors have successfully used these formulas, combine with the algorithm described in \cite{montgomery-smith-parallel}, to create software for controlling a cable-driven parallel robot, which was built by the Dynamic Systems Test Branch of the Software, Robotics, and Simulation Division (ER5) at the NASA Johnson Space Center.  Because of Export Administration Regulations we are unable to provide many more details.

The paper is quite heavy with mathematical formulas.  For this reason, the proofs of most of the statements, and many of the comments of a mathematical nature, relegated to the Appendix, Section~\ref{sec proofs}.

\section{Notation}

Rotations can be represented by unit quaternions \cite{quaternions1,quaternions2}, which we briefly describe here.  A \emph{quaternion} is a quadruple of real numbers, written as $A = w + x \bi + y \bj + z \bk$, with the algebraic operations $\bi^2 = \bj^2 = \bk^2 = \bi \bj \bk = -1$.  Its \emph{conjugate} is $A^* = w - x \bi - y \bj - y \bk$, its \emph{norm} is $|A| = (w^2+x^2+y^2+z^2)^{1/2} = \sqrt{A A^*} = \sqrt{A^* A}$, its \emph{normalization} is $\widehat A = A/|A|$, its \emph{real part} is $\realpart(A) = w = \tfrac12(A + A^*)$, and its \emph{imaginary part} is $\imagpart(A) = \bi x + \bj y + \bk z = \tfrac12(A - A^*)$.  It is called a \emph{unit} quaternion if $|A| = 1$, a \emph{real} quaternion if $\imagpart(A) = 0$, and a \emph{vector} quaternion if $\realpart(A) = 0$.  If $A \ne 0$, the multiplicative inverse is given by $A^{-1} = A^*/|A|^2$.  (Note that many sources use the word ``pure'' instead of ``vector'' in this context.)

We identify three dimensional vectors with vector quaternions, by identifying $\bi$, $\bj$, and $\bk$ with the three standard unit vectors.  A unit quaternion $Q$ represents the rotation which takes the direction $\bm r$ to $Q \bm r Q^*$.  A rotation by angle $a$ about an axis $\bm s$, where $|\bm s| = 1$, has two unit quaternion representations: $\pm(\cos(\tfrac12 a) + \bm s \sin(\tfrac12 a)) = \pm \exp(\tfrac12 a \bm s)$.  Composition of rotations corresponds to multiplication of unit quaternions.

We can represent quaternions as four dimensional vectors, and give it the inner product
\begin{equation}
A \cdot B = \realpart(A B^*) = \realpart(A^* B) .
\end{equation}

A \emph{dual quaternion} is a pair of quaternions, written as $\eta = A + \epsilon B$, with the extra algebraic operation $\epsilon^2 = 0$.  We call $A = \mathcal P(\eta)$ the \emph{primary part} of $\eta$, and $B = \mathcal D(\eta)$ the \emph{dual part} of $\eta$.

The \emph{conjugate} dual quaternion of $\eta = A + \epsilon B$ is $\eta^* = A^* + \epsilon B^*$.  Conjugation reverses the order of multiplication:
\begin{equation}
(\eta_1\eta_2)^* = \eta_2^* \eta_1^* .
\end{equation}
There is another conjugation for dual quaternions: $\overline{A + \epsilon B} = A - \epsilon B$, but we have no cause to use it in this paper, except in equation~\eqref{defn of rigid motion on 3-vector} below.

A \emph{unit} dual quaternion $\eta = Q + \epsilon B$ is a dual quaternion such that $\eta^*\eta = 1$, equivalently, that $Q$ is a unit quaternion and $B \cdot Q = 0$.  A \emph{vector} dual quaternion $A + \epsilon B$ is a dual quaternion such that both $A$ and $B$ are vector quaternions.

If $\eta = A + \epsilon B$ is a dual quaternion with $A \ne 0$, then its multiplicative inverse can be calculated using the formula
\begin{equation}
\eta^{-1} = A^{-1} - \epsilon A^{-1} B A^{-1}.
\end{equation}
If $\eta$ is a unit dual quaternion, then there is a computationally much faster formula (see \cite{adorno}):
\begin{equation}
\label{inverse unit}
\eta^{-1} = \eta^*.
\end{equation}

We set $\setinvertibledualquat$ for the set of invertible dual quaternions (that is, $A +\epsilon B$ where $A \ne 0$), $\setdualquat$ for the set of dual quaternions, $\setunitdualquat$ for the set of unit dual quaternions, and $\setvectordualquat$ for the set of vector dual quaternions.

A rigid motion of the form
\begin{equation}
\label{rigid motion action}
\bm r \mapsto Q \bm r Q^* + \bm t,
\end{equation}
where here $\bm t$ is a translation, can be represented by the unit dual quaternion
\begin{equation}
\label{rigid motion as dual quaternion}
\eta = Q + \tfrac12 \epsilon \bm t Q.
\end{equation}
Composition of rigid motions corresponds to multiplication of unit dual quaternions, where the notation $\eta_1 \eta_2$ means to apply first the rigid motion represented by $\eta_2$, and then by $\eta_1$, that is, the dual quaternion acts by left multiplication.  If $\bm r$ is a 3-vector, and $\bm s$ is the image of $\bm r$ under the action of the rigid motion $\eta = Q + \epsilon B$, then
\begin{equation}
\label{defn of rigid motion on 3-vector}
1 + \epsilon \bm s = \eta (1 + \epsilon \bm r) \overline\eta^* ,
\end{equation}
but generally it is easier to use the formula
\begin{equation}
\label{rigid motion on 3-vector}
\bm s = Q \bm r Q^* + 2 B Q^* = (Q \bm r + 2 B) Q^* .
\end{equation}

For a dual quaternion, it is not really possible to mix its primary and dual parts additively.  For example, for a unit dual quaternion that represents a rigid motion, the primary part is unitless, whereas the dual has units of length.  For this reason, when measuring how large such a dual quaternion is, everything must be with respect to a characteristic length scale $l$.  (For example, for a parallel robot, the characteristic length might be the width of the workspace of the end effector.)  The \emph{size} of a dual quaternion is defined to be
\begin{equation}
\label{size}
\size_{l}(\eta) = \left(|\mathcal P(\eta)|^2 + l^{-2}|\mathcal D(B)|^2\right)^{1/2}.
\end{equation}

A \emph{twist} is the pair of vectors $(\bm w, \bm v)$ that describes the rate of change of pose or rigid motion, where $\bm w$ is the angular velocity, and $\bm v$ is the translational velocity.  One can think of the twist $(\bm w, \bm v)$ as a rigid motion function of time $t$:
\begin{equation}
\bm r \mapsto \bm r + t (\bm w \times \bm r + \bm v) + O(t^2) \quad \text{as $t \to 0$} .
\end{equation}
It has two possible meanings, depending upon whether the twist is understood to be with respect to the fixed frame, or with respect to the moving frame.  If it is understood to be with respect to the moving frame, we have the formula
\begin{equation}
\label{ode twist}
\frac d{dt} (Q \bm r Q^* + \bm t) = Q (\bm w \times \bm r + \bm v) Q^* + \bm t,
\end{equation}
and if it is understood to be with respect to the fixed frame
\begin{equation}
\label{ode twist fixed}
\frac d{dt} (Q \bm r Q^* + \bm t) = \bm w \times (Q \bm r Q^* + \bm t) + \bm v .
\end{equation}
Then this twist can be represented by a vector dual quaternion \cite{adorno, agrawal}
\begin{equation}
\label{twist as dual quaternion}
\varphi = \tfrac12 \bm w + \tfrac12 \epsilon \bm v ,
\end{equation}
where if we understand it to be with respect to the moving frame, we have the formula
\begin{equation}
\label{ode dual quaternion}
\varphi = \eta^{-1} \dot \eta, \quad\text{or}\quad\dot \eta = \eta \varphi ,
\end{equation}
and if we understand it to be with respect to the fixed frame
\begin{equation}
\label{ode dual quaternion fixed}
\varphi = \dot \eta \eta^{-1}, \quad\text{or}\quad\dot \eta = \varphi \eta .
\end{equation}
In this paper, unless otherwise stated, we always understand the twist to be with respect to the moving frame.

We make the identification
\begin{equation}
\setdualquat \cong \mathbb R^8,
\end{equation}
using the basis
\begin{equation}
\label{dual quaternion basis}
(\beta_1, \beta_2, \beta_3, \beta_4, \beta_5, \beta_6, \beta_7, \beta_8) = (\bi, \bj, \bk, \epsilon\bi, \epsilon\bj, \epsilon\bk, 1, \epsilon),
\end{equation}
and similarly, we make the identification
\begin{equation}
\label{vectdualquat ident}
\setvectordualquat \cong \mathbb R^6,
\end{equation}
using the basis $(\beta_1$, $\beta_2$, $\beta_3$, $\beta_4$, $\beta_5$, $\beta_6)$.  With these identifications, we can define the dot product between two dual quaternions by transferring the usual definition of dot product on $\mathbb R^8$, that is
\begin{equation}
(A + \epsilon B) \cdot (C + \epsilon D) = A \cdot C + B \cdot D.
\end{equation}
In this way, every dual quaternion $\eta$ can be written in component form as
\begin{equation}
\eta = \sum_{i=1}^8 \eta_i \beta_i ,
\end{equation}
and every vector dual quaternion $\theta$ as
\begin{equation}
\theta = \sum_{i=1}^6 \theta_i \beta_i .
\end{equation}
Finally, we give a few extra formulas.  Let $\varphi_m$ denote the twist with respect to the moving frame, and $\varphi_f$ denote the twist with respect to the fixed frame.  From
equations~\eqref{ode dual quaternion} and~\eqref{ode dual quaternion fixed} we obtain
\begin{equation}
\varphi_f = \eta \varphi_m \eta^{-1} .
\end{equation}
Since in any algebra we have
\begin{equation}
\frac d{dt} \eta^{-1}  = - \eta^{-1} \dot \eta \eta^{-1},
\end{equation}
we obtain the surprisingly simple formula for the change of frame for acceleration:
\begin{equation}
\dot\varphi_f = \eta \dot\varphi_m \eta^{-1} .
\end{equation}
Note that this does not generalize to higher derivatives, for example, the formula for change of frame for jerk is
\begin{equation}
\ddot\varphi_f = \eta \ddot\varphi_m \eta^{-1} + \eta(\varphi\dot\varphi - \dot\varphi \varphi)\eta^{-1}.
\end{equation}

\section{Dual quaternions to represent wrenches}

Let the pose $\eta$ represent the reference frame that moves with the end effector.  It is not necessary (although it can simplify things) that the center of mass of the end effector coincides with the origin of the moving frame.

The \emph{wrench dual quaternion} is defined to be
\begin{equation}
\label{wrench as quaternion}
\tau = 2 \bm q + 2 \epsilon \bm p,
\end{equation}
where $\bm q$ and $\bm p$ are the torque and force, respectively, applied to the end effector at the origin of the moving frame, measured with respect to the moving frame.

If $\bm r_0$ is the center of mass of the end effector in the moving frame, then the twist about the center of mass is given by
\begin{equation}
\label{twist correction}
\varphi_0 = 
\varphi + \tfrac12\epsilon\bm w \times \bm r_0,
\end{equation}
where $\varphi = \eta^{-1} \dot \eta$, and the wrench applied about the center of mass is
\begin{equation}
\label{torque correction}
\tau_0 = 
\tau + 2 \bm p \times \bm r_0.
\end{equation}

The reason for introducing the factor $2$ in definition~\eqref{wrench as quaternion} is so that the rate of change of work done to the end effector is given by
\begin{equation}
\label{dot h tau varphi}
\frac{d}{dt} \text{(work done)} = \tau \cdot \varphi = \tau_0 \cdot \varphi_0 .
\end{equation}
(The second equality follows from vector identities.)

See \cite{ball} for the origins of the term twist and wrench as pairs of 3-vectors, which are examples of \emph{screws}.  The `work done' formulas are also known as \emph{reciprocal screw relationships}.

\section{The normalization of a dual quaternion}

A \emph{dual number} is anything of the form $a + \epsilon b$, where $a$ and $b$ are real numbers.  The \emph{norm} of a dual quaternion $\eta = A + \epsilon B$ is the dual number defined by the two steps:
\begin{gather}
|\eta|^2 = \eta^* \eta = \eta \eta^* = |A|^2 + 2 \epsilon (A \cdot B) ,\\
\label{norm}
|\eta| = \sqrt{|\eta|^2} = |A| + \epsilon (A \cdot B) / |A| .
\end{gather}
The norm preserves multiplication, that is, if $\eta_1$ and $\eta_2$ are two dual quaternions, then
\begin{equation}
|\eta_1 \eta_2| = |\eta_1| |\eta_2| .
\end{equation}

If $\eta = Q + \epsilon B$ is an invertible dual quaternion, then we define its \emph{normalization} to be the unit dual quaternion
\begin{equation}
\label{normalize}
\widehat \eta = |\eta|^{-1}\eta = \eta|\eta|^{-1}
= Q/|Q| + \epsilon (B/|Q| - (B\cdot Q) Q/|Q|^3) .
\footnote[2]{Note this formula is incorrect in the published version \cite{this}, and corrected in this paper \cite{this-correction}.  The authors are grateful to Brent Koogler for bringing this to our attention.  This correction makes only a small difference to the results reported in Subsection~\ref{s results}.  The corrected results may be found in \cite{this-correction}.}
\end{equation}
(We remark that the normalization of an invertible dual quaternion is used in the computer graphics industry \cite{kavan-et-al, kavan-et-al-2}.)  While this normalization formula might seem initially quite complicated, after thinking about it one can see that it is the simplest projection that enforces $|Q| = 1$ and $B\cdot Q = 0$.

The normalization also satisfies the following properties.
\begin{itemize}
\item If $\eta$ is a unit dual quaternion, then $\widehat\eta = \eta$.
\item Normalization preserves multiplication, that is, if $\eta_1$ and $\eta_2$ are two dual quaternions, then
\begin{equation}
\widehat{\eta_1 \eta_2} = \widehat \eta_1 \widehat \eta_2 .
\end{equation}
\end{itemize}

\section{Notation for three by three matrices}
\label{sec matrix notation}

Let $\mathsf I$ denote the $(3 \times 3)$ identity matrix, and $\mathsf 0$ denote the $(3 \times 3)$ zero matrix.  If $\bm r$ is a 3-vector, then the \emph{Hodge star operator} of $\bm r$ is
\begin{equation}
\starop{\bm r} = \begin{bmatrix}  0   & - r_3 &  r_2 \\
                               r_3 &  0   & - r_1 \\
                              - r_2 &  r_1 &  0 \end{bmatrix} .
\end{equation}
Note that
\begin{equation}
\label{r star to r cross}
\starop{\bm r} \bm s = \bm r \times \bm s .
\end{equation}
If $\bm u$ is a unit vector, then the \emph{projection onto the complement} of the unit vector $\bm u$ is defined by
\begin{equation}
\mathsf P_{\bm u} \bm x = \bm x - (\bm u \cdot \bm x) \bm u .
\end{equation}
Consistent with the identification~\eqref{vectdualquat ident}, if $\mathsf A$, $\mathsf B$, $\mathsf C$, and $\mathsf D$ are $(3 \times 3)$ matrices,  and $\theta = \tfrac12 \bm a + \tfrac12 \epsilon \bm b$, $\psi = \tfrac12 \bm c + \tfrac12 \epsilon \bm d$ are vector dual quaternions, we have
\begin{gather}
\begin{bmatrix}\mathsf A & \mathsf B\end{bmatrix}
\theta = \tfrac12 \mathsf A \bm a + \tfrac12 \mathsf B \bm b , \\
\begin{bmatrix}\mathsf A & \mathsf B\\
\mathsf C & \mathsf D
\end{bmatrix}
\theta =
\tfrac12 \mathsf A \bm a + \tfrac12 \mathsf B \bm b +
\tfrac12 \epsilon\mathsf C \bm a + \tfrac12 \epsilon \mathsf D \bm b, \\
\theta \cdot \begin{bmatrix}\mathsf A & \mathsf B\\
\mathsf C & \mathsf D
\end{bmatrix}
\psi =
\tfrac14 \bm a \cdot \mathsf A \bm c
+ \tfrac14 \bm a \cdot \mathsf B \bm d
+ \tfrac14 \bm b \cdot \mathsf C \bm c
+ \tfrac14 \bm b \cdot \mathsf D \bm d .
\end{gather}

\section{Lie derivatives}
\label{sec lie derivatives}

The notion of the Lie derivative, sometimes in our context called the directional derivative, is a combination of two ideas that may be found in the literature.  First is the concept of a Lie derivative with respect to a vector field \cite{yano}.  Secondly, the definition of the Lie algebra is that it is the vector space of vector fields that are invariant under left multiplication by elements of the Lie group \cite{lee}.  In this way, we can define the Lie derivative of a function with respect to an element of the Lie algebra.  One place in the literature where they are combined is in \cite[Equation~(5), Chapter~II]{helgason}.

These standard abstract definitions can be made more concrete in our special case where the Lie group is the set of unit dual quaternions, and the Lie algebra is the set of vector dual quaternions.

Suppose one has a quantity that is a function of pose $g(\eta)$, and whose range is any vector space.  (But for intuition, consider the case when the range is the real numbers, and think of the vector valued case as the formulas below simply being applied component wise.)

Then we usually think of the derivative of $g(\eta)$ as the Jacobian with respect to the components of $\eta$.  But it really makes more sense to compute the derivative with respect to the components of the perturbation of $\eta$.  The latter is the Lie derivative.

The definition is this.  Given a differentiable function $g$ whose domain is the unit dual quaternions, $\setunitdualquat$, and whose codomain is any vector space over the real numbers, we can extend it arbitrarily to a differentiable function whose domain is an open neighborhood of $\setunitdualquat$ in $\setdualquat$.  Given a unit dual quaternion $\eta$ and a vector dual quaternion $\theta$, we define the \emph{Lie derivative} of $g(\eta)$ in the direction of $\theta$ to be
\begin{equation}
\label{lie diff defn}
\liederiv_\theta g = 
\lim_{r\to 0} \frac{g(\eta(1+r \theta)) - g(\eta)} r = 
\left. \frac{d}{d r} g(\eta(1+r\theta)) \right |_{r = 0}.
\end{equation}
Since $\eta(1+r\theta)$ isn't necessarily a unit dual quaternion, it is not obvious that the definition of the directional Lie derivative doesn't depend upon how the domain of $g$ was extended from $\setunitdualquat$, but it is, as is shown in Lemma~\ref{not depend} below.

Given a generic function $g$ whose domain is the invertible dual quaternions, $\setinvertibledualquat$, and whose codomain is any vector space over the real numbers, we define its Jacobian to be the dual quaternion
\begin{equation}
\label{jacobian defn}
\frac{\partial g}{\partial \eta} = \sum_{i=1}^8 \left[\frac{\partial}{\partial \eta_i} g(\eta) \right] \beta_i,
\end{equation}
where the partial derivative $\dfrac{\partial}{\partial \eta_i}$ is interpreted using the identification of $\eta_i$ with the $\beta_i$ components of $\eta$ as described in \eqref{dual quaternion basis}.

If the domain of $g$ is the vector dual quaternions, $\setvectordualquat$, we have equation~\eqref{jacobian defn} except with $8$ replaced by $6$.

We have the following formula, which is useful for explicitly calculating the Lie derivative if an extension of $g$ to an open neighborhood of $\setunitdualquat$ in $\setdualquat$ is known:
\begin{equation}
\label{lie diff defn 2}
\liederiv_\theta g = \frac{\partial g}{\partial \eta} \cdot (\eta \theta) := \sum_{i=1}^8 \left(\frac{\partial g}{\partial \eta_i}\right) (\eta \theta)_i ,
\end{equation}
where the notation $(\eta\theta)_i$ means the $\beta_i$ component of $\eta\theta$, as described in \eqref{dual quaternion basis}.

We define the \emph{partial Lie derivatives} to be
\begin{equation}
\liederiv_i g = \liederiv_{\beta_i} g , \quad (1 \le i \le 6),
\end{equation}
and its \emph{full Lie derivative} to be the vector dual quaternion (or if the range of $g$ is a vector space, the tensor product of a vector dual quaternion with a vector)
\begin{equation}
\liederiv g = \sum_{i=1}^6 \liederiv_i g \beta_i,
\end{equation}
so that for all vector dual quaternions $\theta$ it satisfies:
\begin{equation}
\theta \cdot \liederiv g = \liederiv_\theta g.
\end{equation}

To gain some intuition, write
\begin{equation}
\label{decomp theta}
\theta = \tfrac12\bm a + \tfrac12\epsilon \bm b.
\end{equation}
Since we have that
\begin{equation}
\label{decomp partial theta}
\frac{\partial}{\partial \theta} = 2\frac{\partial}{\partial \bm a} + 2\epsilon \frac{\partial}{\partial \bm b},
\end{equation}
we see that $\theta$ represents a change in pose by an infinitesimal translation $\bm b$ and an infinitesimal rotation $\bm a$, measured in the moving frame of reference.  Thus $\liederiv g$ is a vector dual quaternion giving twice the change in $g$ with respect to an infinitesimal rotation, plus $\epsilon$ times twice the change of $g$ with respect to an infinitesimal translation.

One important property of the Lie derivative is that if $\eta$ represents a pose, with twist $\varphi$, then by equation~\eqref{lie diff defn}, with $r$ replaced by $t$, we see that
\begin{equation}
\label{dot f Lie phi}
\frac d{dt} [g(\eta)] = \liederiv_\varphi g .
\end{equation}

The Lie derivative satisfies various rules, which easily follow from either equations~\eqref{lie diff defn} or ~\eqref{lie diff defn 2}, which are also useful for explicitly calculating the Lie derivative when $g$ is known.
\begin{itemize}
\item If $g(\eta)$ is linear in $\eta$, then
\begin{equation}
\label{rule linear}
\liederiv_\theta g(\eta) = g(\eta \theta).
\end{equation}
\item The product rule: if $*$ is any binary operator which is bilinear over the real numbers, such as the product of real numbers, the inner product, the cross product,
or the dual quaternion product, then
\begin{equation}
\label{rule product}
\liederiv_\theta (g_1 * g_2) = g_1 * (\liederiv_\theta g_2) + (\liederiv_\theta g_1) * g_2.
\end{equation}
\item The chain rule:
\begin{equation}
\label{rule chain}
\liederiv_\theta (h(g_1, g_2,\dots,g_m)) \\ = \sum_{i=1}^m
\frac\partial{\partial g_i} h(g_1, g_2,\dots,g_m) \liederiv_\theta g_i .
\end{equation}
\item
Let $\tilde{\bm s}$ be a constant position vector, and $\tilde{\bm n}$ be a constant direction.  Let $\bm s$ and $\bm n$ be their corresponding values with respect to the moving frame.  Then
\begin{align}
\label{lie deriv vector 1}
\mathcal L_\theta \bm s &= 2 \begin{bmatrix} \starop{\bm s} & - \mathsf I \end{bmatrix} \theta , \\
\label{lie deriv vector 2}
\mathcal L_\theta \bm n &= 2 \begin{bmatrix} \starop{\bm n} & \mathsf 0 \end{bmatrix} \theta .
\end{align}
\end{itemize}
To simplify the writing of application software, Using these rules, one can create a software library in C++ that performs automatic Lie differentiation.  Since the domain, and hence range, of the Lie derivative can be any vector space, the sensible way to do this is using templates to allow for a variety of different data types.  Note also that the product rule~\eqref{rule product} has to be implemented for every product that is used, and similarly the chain rule~\eqref{rule chain} for every function $h$ that is used.

\section{Applications to parallel robots}

\begin{figure}
\includegraphics[scale=0.4,trim=240 10 140 10,clip]{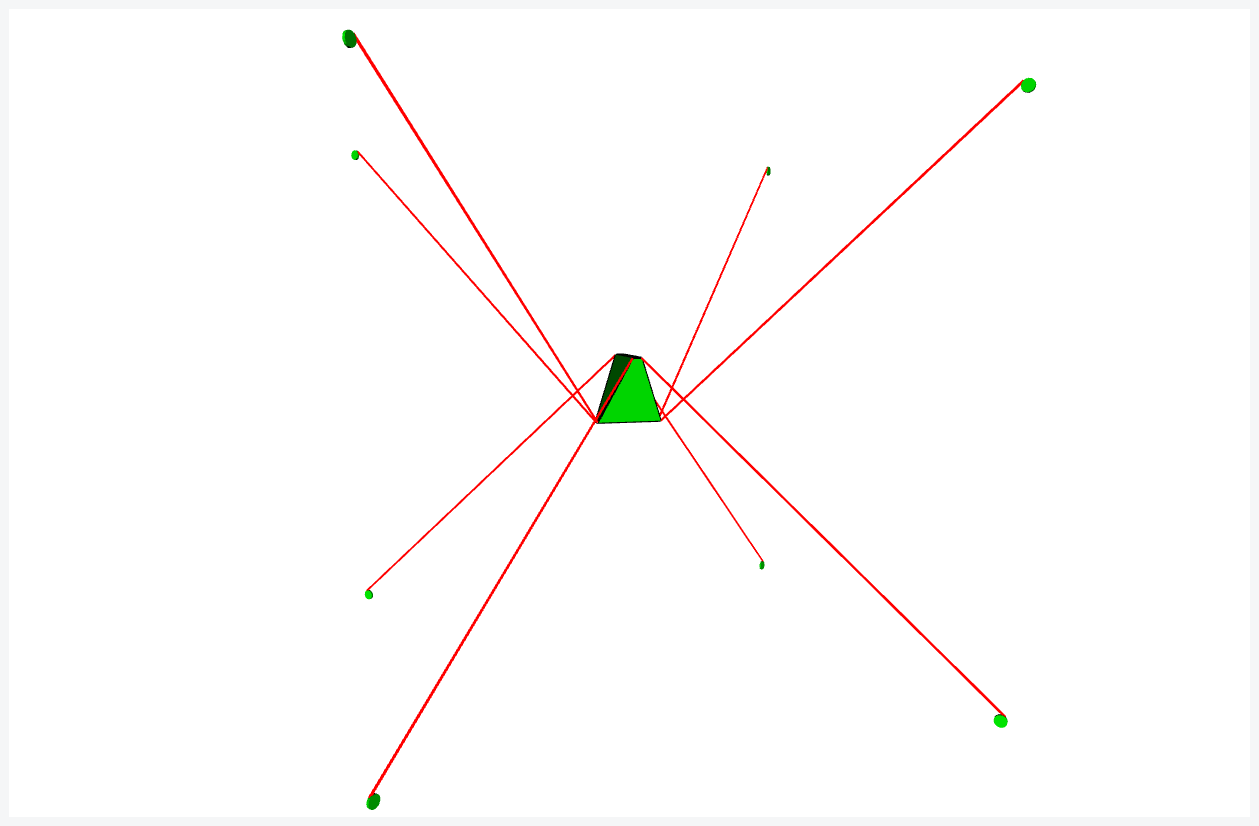}
\caption{Schematic of a cable-driven parallel robot.}
\label{cable-driven parallel robot}
\end{figure}

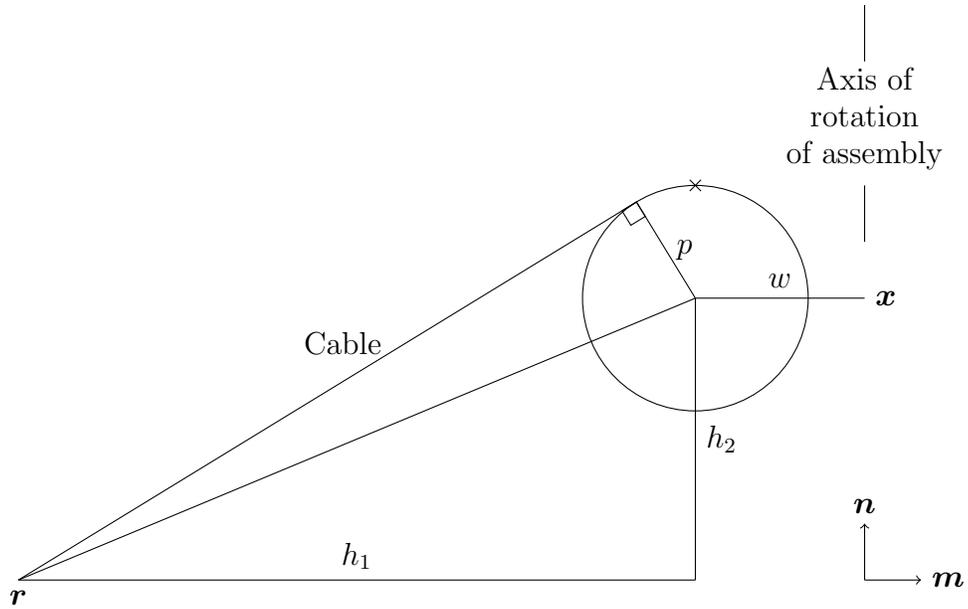
\begin{figure}

\begin{tikzpicture}[scale=0.75]
\node[draw, circle, minimum size=3cm] at (0,0) {};
\draw[-] (0.1,2.1) -- (-0.1,1.9);
\draw[-] (-0.1,2.1) -- (0.1,1.9);
\draw[-] (0,0) -- (-12,-5);
\draw[-] (0,0) -- node[right] {$p$} (-1.0441,1.70583);
\draw[-] (-1.0441,1.70583) -- (-12, -5);
\node[below] at (-12,-5) {$\boldsymbol r$};
\draw[rotate=31.4699] (0,2) rectangle (-0.3,1.7);
\node at (-6.25,-0.8) {Cable};
\draw[-] (-12,-5) -- node[above] {$h_1$} (0,-5);
\draw[-] (3,4.2) -- (3,5.2);
\draw[->] (3,-5) -- (3,-4) node[above] {$\boldsymbol n$};
\draw[->] (3,-5) -- (4,-5) node[right] {$\boldsymbol m$};
\draw[-] (3,1) -- (3,2) node[above] {\minibox[c]{Axis of\\rotation\\of assembly}};
\draw[-] (0,0) -- node[right] {$h_2$} (0,-5);
\draw[-] (3,0) node[right]{$\bm x$} -- node[above]{$w$} (0,0) ;
\end{tikzpicture}

\caption{The pulley and attached cable}
\label{pulley}
\end{figure}

To aid our description, we show the example of a cable-driven parallel robot in Figures~\ref{cable-driven parallel robot} and~\ref{pulley}.  In Figure~\ref{cable-driven parallel robot}, we show the end effector in green, controlled by eight cables shown in red, which go around pulleys shown in green.  Each cable, and how it is attached to the pulley, is shown in Figure~\ref{pulley}.

Suppose that the position of the end effector of a parallel robot is given by $n$ actuators, described by quantities
\begin{equation}
\bm \ell = (\ell_j)_{1 \le j \le n} .
\end{equation}
For example, for a cable-driven parallel robot, these represent the lengths of the cables, and typically $n = 8$.  The numbers $\ell_j$ only need to be determined up to a fixed number.  Thus, for example, in Figure~\ref{pulley}, the number $\ell_j$ could represent the length of cable from the end effector attachment point, to the small cross marked at the top of the pulley.

For the Gough or Stewart Platforms \cite{gallardo-alvarado}, we have $n = 6$, and $\ell_j$ is simply the distance between the end effector attachment point and the ground attachment point.

Let us also denote the force exerted by the actuators by 
\begin{equation}
\bm f = (f_j)_{1 \le j \le n} ,
\end{equation}
defined so that the rate of change of work performed through the actuators is given by
\begin{equation}
\label{dot h f ell}
\frac{d}{dt} \text{(work done)} = \bm f \cdot \dot{\bm\ell} .
\end{equation}

Suppose we have a function $\mathsf L : \setunitdualquat \to \mathbb R^n$, which calculates the required actuator values, $\bm \ell$, from the pose $\eta$ of the end effector frame.  This is the \emph{inverse kinematics} function.

We also define the $(n \times 6)$ matrix $\mathsf \Lambda$ by
\begin{equation}
\label{Lambda}
\mathsf \Lambda \theta = \liederiv_\theta \mathsf L = \liederiv_\theta \bm \ell ,
\end{equation}
or more explicitly, by
\begin{equation}
\mathsf \Lambda_{i,j} = \liederiv_j \mathsf \ell_i .
\end{equation}
From equation~\eqref{dot f Lie phi} we obtain
\begin{equation}
\label{dot l Lambda phi}
\dot{\bm \ell} = \mathsf \Lambda \varphi .
\end{equation}

There is also a ($6 \times n$) matrix $\mathsf T$ that maps the actuator forces to the wrench dual quaternion:
\begin{equation}
\label{T}
\tau = \mathsf T \bm f .
\end{equation}
This can be computed by balancing the force and torque exerted upon the end-effector.  But it can also be computed with the following important identity:
\begin{equation}
\label{T=L^T}
\mathsf T = \mathsf \Lambda^T.
\end{equation}

\section{Second lie derivatives}

If $g$ is a function of dual quaternions with codomain any vector space over the real numbers, we define its Hessian to be the $(8 \times 8)$ matrix
\begin{equation}
\left[\frac{\partial^2 g}{\partial \eta^2}\right] = \left[\frac{\partial^2 g}{\partial \eta_i\partial \eta_j}\right] _{ 1 \le i,j \le 8 }.
\end{equation}
Thus the expression $\left[\frac{\partial^2 g}{\partial \eta^2}\right] \gamma$ should be interpreted as a matrix product with $\gamma$ treated as an eight dimensional column vector.

Second Lie derivatives will be used in Newton's Method, as well as in our statements of the equations of motion, both described below.

We have that
\begin{equation}
\label{second lie deriv}
\begin{aligned}
\liederiv_\theta \liederiv_\psi g &= (\eta \psi) \cdot \left[\frac{\partial^2 g}{\partial \eta^2}\right](\eta \theta) + \frac{\partial g}{\partial \eta} \cdot (\eta \theta \psi) \\
&:= \sum_{i=1}^8 \sum_{j=1}^8 (\eta \psi)_i \left(\frac{\partial^2 g}{\partial \eta_i \partial \eta_j}\right)(\eta \theta)_j + \sum_{i=1}^8 \left(\frac{\partial g}{\partial \eta_i}\right) (\eta \theta \psi)_i .
\end{aligned}
\end{equation}

Another way one might try to define the second derivative is to use the formula $\left.\frac {\partial^2}{\partial\theta^2} g(\eta(1+\theta)) \right | _{\theta=0}$.  Unfortunately, this definition doesn't work, as it depends upon the choice of how to extend the domain of $g$ to all dual quaternions.  The obvious choice of extension is to use the normalization $\tilde g(\eta) = g(\hat\eta)$.  We have the following formula for the Hessian of $\tilde g$:
\begin{equation}
\label{hessian}
\left.\frac{\partial^2}{\partial\theta^2} \tilde g(\eta(1+\theta)) \right|_{\theta = 0} = \left[\tfrac12 (\liederiv_i \liederiv_j g(\eta) + \liederiv_j \liederiv_i g(\eta)) \right]_{1\le i,j\le 6}.
\end{equation}

\section{The examples of a Stewart Platform, and a cable-driven parallel robot}
\label{section examples}

As an example, let us consider parallel robots such as cable-driven robots, or Stewart Platforms.  In this case, the end effector has certain `attachment points' on it, $\bm r_1$, $\bm r_2,\dots$, $\bm r_n$ where the cables or legs attach, the cables or legs are attached at the other end to $\bm s_k$,
and unit vectors $\bm u_1$, $\bm u_2,\dots$, $\bm u_n$ which are directions the cables or legs come into the end effector, all of these being measured in the end effector's frame of reference.  Note that for $1 \le k \le n$
\begin{equation}
\bm u_k = \frac{\bm r_k - \bm s_k}{|\bm r_k - \bm s_k|} .
\end{equation}
Then
\begin{equation}
\label{L ell_k}
\mathcal L \ell_k = 2 \bm r_k \times \bm u_k + 2 \epsilon \bm u_k = 2 \begin{bmatrix} \starop{\bm r_k} \\ \mathsf I \end{bmatrix} \bm u_k,
\end{equation}
or representing vectors as columns, we have
\begin{equation}
\label{Lambda stewart platform}
\mathsf\Lambda = 2 \begin{bmatrix}
(\bm r_1 \times \bm u_1)^T & \bm u_1^T \\
(\bm r_2 \times \bm u_2)^T & \bm u_2^T \\
\vdots & \vdots \\
(\bm r_n \times \bm u_n)^T & \bm u_n^T \\
\end{bmatrix} .
\end{equation}

For calculating the second Lie derivative of $\ell_k$, we need only know the first Lie derivative of $\bm u_k$.  Note that in the simple case that the attachment point of the cable to the frame is fixed in the fixed frame (for example, as in a Stewart Platform), we have
\begin{equation}
\label{L u_k - fixed}
\mathcal L \bm u_k =
\frac1{|\bm r_k - \bm s_k|} \ 
\mathsf P_{\bm u_k}
\begin{bmatrix} - \starop{\bm s_k} & \mathsf I
\end{bmatrix} ,
\end{equation}
and therefore
\begin{equation}
\label{L^2 ell_k - fixed}
\mathcal L^2 \ell_k =
2 \begin{bmatrix} \starop{\bm r_k} \\ \mathsf I \end{bmatrix}
\frac1{|\bm r_k - \bm s_k|} \ 
\mathsf P_{\bm u_k}
\begin{bmatrix} - \starop{\bm s_k} & \mathsf I
\end{bmatrix} .
\end{equation}

Let us also describe a more complicated situation, which matches the cable-driven parallel robot the NASA Johnson Space Center described in the introduction.  For simplicity of notation, we drop the subscripts $k$.  See Figure~\ref{pulley} for reference.  Let us suppose that the cable attaches via an assembly, which is free to rotate about an axis parallel to the unit vector $\bm n$, and passing through the point $\bm x$.  Attached to the assembly, at a fixed distance $w$ from the fixed point $\bm x$, by a rod perpendicular to $\bm n$, is the center of a pulley of radius $p$, which rotates in the plane containing the axis of rotation and the point on the end effector $\bm r$.  A cable passes over the pulley in the $\bm n$ direction from the center of the pulley.  All of the vectors are expressed in the moving frame of reference.

We work in the $(\bm m, \bm n)$ coordinate system in which the origin is $\bm r$.  Let $(h_1,h_2)$ be the coordinates of the center of the pulley.  Then
\begin{equation}
\bm u = u_1 \bm m + u_2 \bm n,
\end{equation}
where
\begin{equation}
(u_1, u_2) = - \frac1{(h_1^2+h_2^2)}
\begin{bmatrix}
h_1 & -h_2 \\
h_2 & h_1
\end{bmatrix}
\begin{bmatrix}
\sqrt{h_1^2+h_2^2 - p^2} \\ p
\end{bmatrix}.
\end{equation}
The cable length, up to an additive constant, is
\begin{equation}
\ell = \sqrt{h_1^2 + h_2^2 - p ^2} + p \tan^{-1}\left(\frac{u_2}{u_1}\right).
\end{equation}
To avoid singularities, the optimal way to compute the inverse tangent, at least with the configuration shown in Figure~\ref{pulley}, is to use $\text{atan2}(-u_2, -u_1)$, where the commonly available $\text{atan2}(y,x)$ function solves for $\theta$ where $x = r \cos\theta$, $y = r \sin\theta$, $ r > 0$, $-\pi < \theta \le \pi$.

The various required quantities are
\begin{gather}
\bm d = \mathsf P_{\bm n} (\bm x - \bm r) , \\ 
\bm m = \widehat{\bm d} = \frac{\bm d}{|\bm d|}, \\
h_2 = \bm n \cdot (\bm x - \bm r) , \\
h_1 = \bm m \cdot (\bm x - \bm r) - w = \sqrt{|\bm x - \bm r|^2 - h_2^2} - w .
\end{gather}
The Lie derivatives of these quantities can be calculated using the automatic Lie differentiation described at the end of Section~\ref{sec lie derivatives}.  The only second Lie derivative required is that of $\ell$, and since we have equation~\eqref{L ell_k}, no automatic second Lie differentiation is required.

\subsection{Singularity analysis for Stewart Platforms}

When operating a Stewart Platform, a singularity occurs when there are no viable cable forces that can create an arbitrary wrench, or equivalently, when there exists infinitesimal perturbations of the end effector pose that don't require a leg length change.  These singularities are often called bifurcations, because after a Stewart Platform encounters a singularity, the end effector is free to move in more than one direction.  Encountering a singularity can cause great damage to the Stewart Platform.

From these considerations, it becomes clear that a singularity happens for a Stewart Platform if and only if $\det(\mathsf \Lambda) = \det(\mathsf T) = 0$.  This is in agreement with the results obtained by Gosselin and Angeles \cite{gosselin-angeles}.

Note that when considering more complex parallel robots, that a singularities can happen in other situations as well (see, for example, \cite{merlet}).

\section{Forward kinematics}
\label{sec forward kinematics}

In robotics, there are several ways to find the pose of the end effector.  One method is to use an optical system, but this is not very accurate.  Another is to use a proprioceptive sensor, where the pose is found by integrating the acceleration and angular velocity of the end effector, but this is subject to drift.  It would be extremely helpful if the end effector could be computed from the numbers $\ell_k$.  These numbers lengths can be found easily and with high sampling frequency using, for example, encoders.

Let the \emph{set of admissible actuator values}, $\mathbb L \subset \mathbb R^n$, be the range of the function $\mathsf L$.  Then the \emph{forward kinematics} function is
\begin{equation}
\mathsf Y : \mathbb L \to \setunitdualquat,
\end{equation}
 which is a left inverse to $\mathsf L$.  Because of possible measurement errors, $\mathsf Y$ should produce decent answers even if the actuator values are merely close to $\mathbb L$.

The methods described here are essentially the Newton-Raphson Method, and are all iterative methods.  Given a guess $\eta_k$, we create a new guess $\eta_{k+1}$:
\begin{equation}
\eta_{k+1} = \eta_k \widehat{(1 + \theta_{k+1})},
\end{equation}
and iterate until some criterion is met.  We measure $\size_l(\eta_{k+1} - \eta_k)$, and see when it is smaller than some preset value, like $10^{-16}$.

Here we merely describe the method.  In Section~\ref{sec proofs}, we explain why they work.

We would also like to mention a different approach to using dual quaternions to solve forward kinematics problems in \cite{yang-et-al}, although we don't think it will cover the more complicated situation described in Section~\ref{section examples}.

\subsection{Forward kinematics for Stewart Platforms}
\label{forward kinematics stewart}

First we describe how to solve the exact-constrained problem, that is, when $n= 6$.  This would be the case for a Stewart Platfom.  Typically this is solved by writing the pose using Euler angles, which provides a way to represent the pose using a 6-vector.  However, in the opinion of the authors, this becomes a rather complicated set of equations, resulting in quite lengthy code.

Our algorithm is
\begin{equation}
\label{exact constrained Newton-Raphson}
\theta_{k+1} = - \mathsf \Lambda(\eta_k)^{-1} (\mathsf L(\eta_k) - \bm\ell) .
\end{equation}
This is easy to code, certainly simpler than methods which use Euler angles.

\subsection{Forward kinematics for over-constrained parallel robots}
\label{forward kinematics}

Next we focus on the over-constrained problem, that is, when the number of actuators $n$ is greater than $6$.

This problem has been solved by many others, for example, \cite{pott-schmidt,yang-et-al}.  But we feel that this is more easily solved using dual quaternions.  For example, using the programming language C++, one can quickly build classes representing dual quaternions, and then these formulas can be applied without any real thought.

We compute
\begin{equation}
\delta_k = \left(\sum_{m=1}^n \mathsf \Lambda_{m,i} (\mathsf L_m(\eta_k) - \ell_m)\right)_{1 \le i \le 6},
\end{equation}
and
\begin{equation}
\mathsf H_k
= \left[\sum_{m=1}^n \mathsf\Lambda_{m,i} \mathsf\Lambda_{m,j} + \tfrac12 \mathcal L_i \mathsf \Lambda_{m,j} (\mathsf L_m(\eta_k)) + \tfrac12 \mathcal L_j \mathsf \Lambda_{m,i} (\mathsf L_m(\eta_k))
\right]_{1\le i,j\le 6}.
\end{equation}
The algorithm is
\begin{equation}
\label{over constrained Newton-Raphson}
\theta_{k+1} =  \mathsf H_k^{-1} \delta_k .
\end{equation}

\subsection{Results of simulations for forward kinematics}
\label{s results}

The software used to check these algorithms described here is currently proprietary, and so we cannot give too many details.  We hope to get permission to release the software at a later date, and make it available to everyone.

To test the simple Stewart Platform algorithm, we created 10,000 random poses.  The orientation of each pose had an angle no greater than $30^\circ$ from the identity pose.  For each pose, the cable lengths were calculated.  The Newton-Raphson Method was then applied to the cable lengths, with a random initial guess pose.

All runs were successful.  The average number of required iterations was about 4.8.  The run time for each forwards kinematics calculation was a little under 100 microseconds, using a fairly modern but low-end laptop.  Increasing the allowed angle to $45^\circ$ gave a failure rate of 2 in 10,000.

The Newton-Raphson Method for the over-constrained robot given by the more complicated situation described in Section~\ref{section examples} was more delicate.  In particular, since it is minimizing a loss function rather than directly solving the problem, it is possible that it might find local minima of the loss function which didn't correspond to the solution.

For the first test, we created 1,000 random poses, and computed their cable lengths.  The Newton-Raphson Method was applied, with a random initial guess pose, and was allowed up to 50 iterations.  If the loss function of the final answer was greater than $10^{-16}$, the Newton-Raphson Method was applied again.

All poses were found, but the average number of times the Newton-Raphson had to be applied was about 80.  The average time to find a pose was about 5 milliseconds.

For the second test, we again created 1,000 random poses and computed their cable lengths.  Then the Newton-Raphson was applied, with an initial guess that was 1\% different from the original pose.  Again success was measured by computing the loss function.  But there were no second chances.

All poses were found, the average number of iterations required was 4.2, and the average time taken was a little under 100 microseconds.

Increasing the allowable difference between the initial guess and the original pose to 5\% resulted in only about 88\% of the runs being successful.

Note that in all these trials, since the poses were randomly created, it is quite likely that some of the poses were non-feasible for the parallel robots under consideration.  (For example, it might be impossible to maintain that pose while keeping the cables under tension, or moving to that pose might require crossing cables.)

When using this algorithm, to find the initial pose, we suggest to use the first method of trying 1,000 different initial guesses, and choosing the pose with the lowest loss function.  But thereafter, sample the cable lengths often, and use the previously measured pose as the initial guess.  In our numerical simulations, the pose changes by about 0.3\% per time step (about 1 millisecond), and the algorithm has never failed to solve the forward kinematics problem.

\section{Dynamics of the end effector}
\label{dynamics}

The dynamics equations of motion of rigid bodies is well known.  But in this section, we also consider the additional effect of overcoming inertia from the actuators.  Furthermore, we feel that it is nice to see this stated and derived in the context of dual quaternions.

Let us suppose that the kinetic energy of the parallel robot is given by
\begin{equation}
\label{ke}
e = \tfrac12 \varphi \cdot \mathsf M \varphi,
\end{equation}
where $\mathsf M$ is a $(6 \times 6)$ positive definite matrix, which depends only upon $\eta$, and which we call the \emph{effective mass of the parallel robot}.

We define the \emph{no-load forces} to be the actuator forces if there is no end effector present:
\begin{equation}
\bm f_0 = \mathsf M_0 \ddot{\bm \ell},
\end{equation}
where $\mathsf M_0$ is a positive definite $(n\times n)$ matrix denoting what we shall call the \emph{effective no-load mass of the actuators}.  This might be caused by, for example, the reflected moment of inertia of the motor that drives each actuator, in which case $\mathsf M_0$ is simply a constant multiple of the identity.  Since $\tfrac12 \dot{\bm\ell} \cdot \mathsf M_0 \dot{\bm\ell} = \tfrac12 \varphi \cdot \mathsf\Lambda^T \mathsf M_0 \mathsf \Lambda \varphi$ is part of the kinetic energy, it follows that $\mathsf M - \mathsf \Lambda^T \mathsf M_0 \mathsf \Lambda$ is a positive semi-definite matrix.

If $m_e$ is the mass of the end effector, $\mathsf M_e$ is the moment of inertia tensor of the end effector about its center of mass, and $\bm r_0$ is the center of mass of the end effectors, all measured with respect to the moving frame, then
\begin{equation}
\label{ke written in full}
e = \tfrac12 m_e |\bm v + \bm w \times \bm r_0|^2 + \tfrac12 \bm w \cdot \mathsf M_e \bm w + \tfrac12 \dot{\bm\ell} \cdot \mathsf M_0 \dot{\bm\ell} , 
\end{equation}
that is
\begin{equation}
\label{M example}
\mathsf M = 4 \begin{bmatrix} \mathsf M_e - m_e \starop{\bm r_0}^2 & m_e \starop{\bm r_0} \\ - m_e \starop{\bm r_0} & m_e \mathsf I \end{bmatrix} + \mathsf\Lambda^T \mathsf M_0 \mathsf\Lambda.
\end{equation}

\begin{Theorem}
\label{equation of motion}
If the kinetic energy satisfies equation~\eqref{ke} with equation~\eqref{M example} holding, and the potential energy $v$ is calculated in the usual manner from the mass of the end effector in a constant gravitational field whose value is $\bm g$ measured with respect to the moving frame, then the equation of motion is
\begin{equation}
\tau = \mu + \mathsf M \alpha,
\end{equation}
where
\begin{multline}
\label{tau example}
\mu = 2\bm w \times (\mathsf M_e \bm w) + 2\epsilon m_e \bm w \times \bm v \\
+ 2 m_e ((\bm w \cdot \bm r_0) (\bm w \times \bm r_0) + \bm r_0 \times (\bm w \times \bm v) + \epsilon \bm w \times (\bm w \times \bm r_0)) \\
+ \mathsf \Lambda^T \mathsf M_0 (\liederiv_\varphi \mathsf \Lambda) \varphi 
- 2 m_e (\bm r_0 \times \bm g + \epsilon \bm g) ,
\end{multline}
and
\begin{equation}
\mathsf M \alpha = 2 \mathsf M_e \dot{\bm w} + 2 m_e (\bm r_0 \times \dot{\bm v}) + 2 \epsilon m_e \dot{\bm v} + 2 \epsilon m_e (\dot{\bm w} \times \bm r_0) + \mathsf\Lambda^T \mathsf M_0 \mathsf \Lambda \alpha.
\end{equation}
\end{Theorem}

The various terms in equation~\eqref{tau example} can be interpreted as follows.
\begin{itemize}
\item $\mathsf M_e \dot{\bm w}$ and $m_e \dot{\bm v}$ are inertial resistance to change of angular and translational velocities.
\item $m_e \bm w \times \bm v$ is the centripetal force required to rotate and move at the same time.
\item $\bm w \times (\mathsf M_e \bm w)$ is the precession torque (so that if the moment of inertia is not isotropic, then the body spins in a counter-intuitive manner, see, for example, \cite{kawano-et-al}).
\item
\begin{equation}
\mathsf \Lambda^T \mathsf M_0 (\liederiv_\varphi \mathsf \Lambda) \varphi + \mathsf\Lambda^T \mathsf M_0 \mathsf \Lambda \alpha = \mathsf T \bm f_0
\end{equation}
is the wrench required to move the actuators, where the no-load forces may be computed using
\begin{equation}
\bm f_0 = \mathsf M_0 \liederiv_\varphi(\mathsf\Lambda \varphi) + \mathsf M_0 \mathsf\Lambda \alpha .
\end{equation}
\item $m_e \bm g$ is the force due to gravity.
\item All terms containing $\bm r_0$ are corrections required since the center of gravity isn't necessarily the same as the origin of the moving frame of reference.  They could be derived by first finding the equations of motion when $\bm r_0 = 0$, and then applying equations~\eqref{twist correction} and~\eqref{torque correction}.
\end{itemize}

\subsection{Numerical verification of the dynamics equations}

The way we numerically verified these equations was by running a simulated motion, and then calculating the work done in three ways.  The first method was to integrate the inner product of the twist and the wrench on the end effector.  The second method was to integrate the sum of the actuator force times the rate of change of the length of the actuator.  The third method was to calculate the total kinetic energy using equation~\eqref{ke written in full}, and the potential energy using the standard mass times gravity time height formula.  The numerical simulations gave the same results for all three methods up to machine precision.

\section{Appendix: Proofs}
\label{sec proofs}

\begin{Lemma} 
\label{lemma approx normalize}
If $\theta$ is a vector dual quaternion, then
\begin{equation}
\label{approx normalize}
\widehat {(1+\theta)} = 1 + \theta + \tfrac12 \theta^2 + O(\theta^3).
\end{equation}
\end{Lemma}

We remark that since the exponential of a unit dual quaternion $\theta$ \cite{wang-et-al} satisfies
\begin{equation}
\label{exponential}
\exp(\theta) = \sum_{k=0}^\infty \frac{\theta^k}{k!},
\end{equation}
then by equation~\eqref{approx normalize} we have
\begin{equation}
\label{exp-normal}
\exp(\theta) = \widehat{(1+\theta)} + O(\theta^3) .
\end{equation}

Let us clarify the big-Oh notation.  We say that
\begin{equation}
\eta(\theta) = O(\gamma(\theta))
\end{equation}
if there exists a constant $c>0$ such that if $\size_l(\theta)$ is sufficiently small, then
\begin{equation}
\size_l( \eta(\theta)) \le c \size_l ( \gamma(\theta)) .
\end{equation}
Let us show that this definition does not depend upon the characteristic length $l$.  In \cite{montgomery-smith-functional-calculus}, we show how to define the `functional calculus' of dual quaternions.  That is, given a continuously differentiable function $f:\mathbb C \to \mathbb C$, satisfying $f(\bar z) = \overline{f(z)}$, we can make sense of $f(\eta)$ for any dual quaternion $\eta = A + \epsilon B$.  First we can define $f$ on quaternions by realizing that any unit vector $\bm n$ satisfies $\bm n^2 = -1$, and hence one merely treats $\bm n$ as an imaginary unit.  Next, if $B$ is decomposed as $B_1 + \bm b_2$, where $B_1$ commutes with $A$, and $\bm b_2$ is a vector quaternion that anti-commutes with $\imagpart(A)$, and setting
\begin{gather}
f_x(z) = \lim_{h \to 0, h \in \mathbb R} \frac{f(z+h)} h , \\
f_{iy}(z) = \lim_{h \to 0, h \in \mathbb R} \frac{f(z+ih)} {ih} , \\
h_f(z) = \begin{cases}\dfrac{f(z) - f(\bar z)}{z - \bar z} & \text{if $\realpart(z) \ne 0$} \\
f_{iy}(z) & \text{if $\realpart(z) = 0$,} \end{cases}
\end{gather}
then we have
\begin{equation}
f(A + \epsilon B) = f(A) + \epsilon f_x(A) \realpart(B_1) + \epsilon f_{iy}(A) \imagpart(B_1) + \epsilon h_f(A) \bm b_2.
\end{equation}
This means that for any functions $f$ and $g$, that $\size_l(f(A + \epsilon B)) \le \size_l(g(A + \epsilon B))$ for all quaternions $B$ if and only if $|f(A)| \le |g(A)|$, $|f_x(A)| \le |g_x(A)|$, $|f_{iy}(A)| \le |g_{iy}(A)|$, and $|h_f(A)| \le |h_g(A)|$.  In particular, the comparison of size does depend upon which characteristic length $l$ is used.

\begin{proof}[Proof of Lemma~\ref{lemma approx normalize}]  Since $\theta^* = -\theta$, we have
\begin{equation}
|1+\theta|^2 = 1 - \theta^2.
\end{equation}
Hence using Taylor's series
\begin{equation}
|1+\theta|^{-1} = 1 + \tfrac12 \theta^2 + O(\theta^3)
\end{equation}
from which it follows that
\begin{equation}
\widehat{(1+\theta)} = (1+\theta)(1+\theta)^{-1} = 1 + \theta + \tfrac12 \theta^2 + O(\theta^3)  .
\end{equation}
\end{proof}

\begin{Lemma}
\label{not depend}
The definition of $\liederiv_\theta g$ in equation~\eqref{lie diff defn} does not depend upon the extension of $g$ from $\setunitdualquat$ to a neighborhood of $\setunitdualquat$ in $\setinvertibledualquat$.
\end{Lemma}

\begin{proof}  First note that from equation~\eqref{approx normalize} we have
\begin{equation}
\label{diff normalize}
\left.\frac d{dr} \widehat{(1 + r \theta)} \right|_{r=0} = \theta.
\end{equation}
Let $g_1$ and $g_2$ be two extensions of $g$ from $\setunitdualquat$ to a neighborhood of $\setunitdualquat$ in $\setinvertibledualquat$.  Define
\begin{equation}
\tilde g(\eta) = g_1(\widehat\eta) = g_2(\widehat\eta).
\end{equation}
Then
\begin{equation}
\begin{aligned}
\left.\frac{d}{d r} \tilde g(\eta(1+r\theta)) \right |_{r = 0}
&=
\left.\frac{d}{d r} g_1(\eta\widehat{(1+r\theta)}) \right |_{r = 0} \\
&=\left.\frac{d}{d r} g_1(\eta(1+r\theta)) \right |_{r = 0},
\end{aligned}
\end{equation}
where the second equality follows from the chain rule and equation~\eqref{diff normalize}.  Similarly for $g_2$.
\end{proof}

\begin{proof}[Proof that Definition~\eqref{lie diff defn} implies Equation~\eqref{lie diff defn 2}]

Using the chain rule for partial derivatives, we obtain
\begin{equation}
\begin{aligned}
\frac{d}{d r} g(\eta(1+r\theta))
&= \sum_{i=1}^8 \left(\frac{\partial g}{\partial \eta_i}(\eta(1+r\theta))\right) \frac{d}{d r} (\eta(1+r\theta))_i 
\\
&= \sum_{i=1}^8 \left(\frac{\partial g}{\partial \eta_i}(\eta(1+r\theta))\right) (\eta\theta)_i .
\end{aligned}
\end{equation}
Now set $r = 0$.
\end{proof}

\begin{proof}[Proof that Equation\eqref{decomp theta} implies Equation~\eqref{decomp partial theta}]
Equation~\eqref{decomp theta} can be written as
\begin{equation}
\bm a = 2(\theta_1 \bm i + \theta_2 \bm j + \theta_3 \bm k), \quad
\bm b = 2(\theta_4 \bm i + \theta_5 \bm j + \theta_6 \bm k),
\end{equation}
where $\theta_i$ is the $\beta_i$ component of $\theta$ as described in \eqref{dual quaternion basis}.  If $f$ is any function of the vector dual quaternions, we have
\begin{equation}
\begin{aligned}
\frac{\partial}{\partial \theta_i} f(\theta)
&= \sum_{j=1}^3 \frac{\partial \bm a_j}{\partial \theta_i} \frac{\partial}{\partial \bm a_j} f(\theta)
+ \sum_{j=1}^3 \frac{\partial \bm b_j}{\partial \theta_i} \frac{\partial}{\partial \bm b_j} f(\theta)
\\&
= \begin{cases}
2 \dfrac{\partial}{\partial \bm a_i} f(\theta) & \text{if $1 \le i \le 3$} \\ \\
2 \dfrac{\partial}{\partial \bm b_{i-3}} f(\theta) & \text{if $4 \le i \le 6$.}
\end{cases}
\end{aligned}
\end{equation}
\end{proof}

\begin{proof}[Proof of Equations~\eqref{lie deriv vector 1} and~\eqref{lie deriv vector 2}]
We have
\begin{align}
\tilde{\bm s} &= Q \bm s Q^* + 2 B Q^* , \\
\tilde{\bm n} &= Q \bm n Q^*,
\end{align}
or
\begin{align}
\bm s &= Q^* (\tilde{\bm s} Q - 2 B ) , \\
\bm n &= Q^* \tilde{\bm n} Q .
\end{align}
Now if $\theta = \tfrac12 \bm a + \tfrac12\epsilon\bm b$, then
\begin{equation}
\liederiv_\theta(\eta) = \eta\psi = \tfrac12 Q \bm a + \tfrac12 \epsilon (Q \bm b + B \bm a),
\end{equation}
that is
\begin{gather}
\liederiv_\theta(Q) = \tfrac12 Q \bm a, \\
\liederiv_\theta(B) = \tfrac12 (Q \bm b + B \bm a) .
\end{gather}
Remembering $\bm a^* = -\bm a$, and that $\bm a \bm s - \bm s \bm a = 2 \,\bm a \times \bm s$, we obtain
\begin{align}
\mathcal L_\theta \bm s &= -\bm a \times \bm s - \bm b ,\\
\mathcal L_\theta \bm n &= - \bm a \times \bm n .
\end{align}
\end{proof}

\begin{proof}[Proof of Equation~\eqref{T=L^T}]
The rate of change of work done on the parallel robot can be computed in two different ways, either using equation~\eqref{dot h tau varphi}, or~\eqref{dot h f ell}.  Substituting in equations~\eqref{dot l Lambda phi} and~\eqref{T}, we obtain
\begin{equation}
\mathsf T \bm f \cdot \varphi = \bm f \cdot \mathsf \Lambda \varphi = \mathsf\Lambda^{T} \bm f \cdot \varphi ,
\end{equation}
the last equality being a standard formula for the transpose of a matrix.  Since this is true for arbitrary actuator forces $\bm f$ and end effector twists $\varphi$, the result follows.
\end{proof}

\begin{proof}[Proof of Equation~\eqref{second lie deriv}]
Applying the rules given in Section~\ref{sec lie derivatives}, we obtain
\begin{align}
\liederiv_\theta \liederiv_\psi g
&= 
\liederiv_\theta \left( \frac{\partial g}{\partial \eta} \cdot (\eta \psi) \right) \\
&= 
\sum_{i=1}^8 \liederiv_\theta\left(\frac{\partial g}{\partial \eta_i}\right) (\eta \psi)_i +
\frac{\partial g}{\partial \eta} \cdot \liederiv_\theta (\eta \psi) \\
&= 
\sum_{i=1}^8 \left( \sum_{j=1}^8\left(\frac{\partial^2 g}{\partial \eta_j \eta_i}\right)(\eta\theta)_j\right) (\eta \psi)_i +
 \frac{\partial g}{\partial \eta} \cdot (\eta \psi \theta) \\
&= (\eta \psi) \cdot \left[\frac{\partial^2 g}{\partial \eta^2}\right](\eta \theta) + \frac{\partial g}{\partial \eta} \cdot (\eta \theta \psi) .
\end{align}
\end{proof}

As a corollary to equation~\eqref{second lie deriv}, we obtain the well known identity:
\begin{equation}
\label{lie bracket diff = diff lie bracket}
\liederiv_\theta \liederiv_\psi g - \liederiv_\psi \liederiv_\theta g = \liederiv_{(\theta\psi - \psi\theta)} g,
\end{equation}
which implies that Lie derivatives do not necessary commute.

\begin{proof}[Proof of Equation~\eqref{hessian}]
We wish to find the Jacobian $\delta$ and the Hessian $\mathsf H$ of $\tilde b$ at the origin, where
\begin{equation}
\tilde b(\theta) = b(\eta  \widehat{(1 + \theta)}).
\end{equation}
We can find this by considering its Taylor series expansion
\begin{equation}
\tilde b(\theta) = \tilde b(0) + \sum_{i=1}^6 \delta_{i} \theta_i + \tfrac12 \sum_{i,j=1}^6 \mathsf H_{i,j} \theta_i \theta_j + O(\theta^3),
\end{equation}
where
\begin{equation}
\theta = \sum_{i=1}^6 \theta_i \beta_i.
\end{equation}
Using the Taylor series, and using equation~\eqref{approx normalize}, one obtains
\begin{equation}
\begin{aligned}
&\tilde b(\theta) = b(\eta  \widehat{(1 + \theta)})
= b(\eta(1 + \theta + \tfrac12 \theta^2 + O(\theta^3))) \\
&= b(\eta) + \frac{\partial b}{\partial \eta} \cdot(\eta(\theta + \tfrac12 \theta^2)) \\
&\phantom{{}={}} + \tfrac12 (\eta \theta) \cdot \left[\frac{\partial^2 b}{\partial \eta^2}\right] (\eta\theta)
+ O(\theta^2) \\
&= b(\eta) + \frac{\partial b}{\partial \eta} \cdot(\eta(\theta)) \\
&\phantom{{}={}} + \tfrac12 (\eta \theta) \cdot \left[\frac{\partial^2 b}{\partial \eta^2}\right] (\eta\theta) + \frac{\partial b}{\partial \eta} \cdot(\tfrac12 \eta \theta^2) + O(\theta^3) .
\end{aligned}
\end{equation}
Now by comparing coefficients, and considering equations~\eqref{lie diff defn 2} and~\eqref{second lie deriv}, we obtain
\begin{equation}
\delta_i = \liederiv_i b(\eta), \quad (1 \le i \le 6),
\end{equation}
and
\begin{equation}
\begin{aligned}
\mathsf H_{i,j} &=
(\eta \beta_i) \cdot \left[\frac{\partial^2 b}{\partial \eta^2}\right] (\eta\beta_j) + \frac{\partial b}{\partial \eta} \cdot(\tfrac12 \eta(\beta_i\beta_j+\beta_j\beta_i)) \\
&= \tfrac12 (\liederiv_i \liederiv_j b(\eta) + \liederiv_j \liederiv_i b(\eta)),
\quad (1\le i,j\le 6).
\end{aligned}
\end{equation}
\end{proof}

\begin{proof}[Proof of Equation~\eqref{L ell_k}]
Suppose that the end effector is moving with pose $\eta = Q + \epsilon B$, and twist $\phi = \tfrac12 \bm w + \tfrac12 \epsilon \bm v$, then with respect to the fixed frame of reference, the velocity of the attachment point $\bm r_k$ is $Q (\bm v + \bm w \times \bm r_k)Q^*$.  And if a force $f_k$ is applied along the direction $\bm u_k$, then the force applied to the end-effector is $f_k Q\bm u_kQ^*$ with respect to the fixed frame of reference.

Hence computing the rate of change of virtual work, we obtain
\begin{equation}
f_k \dot \ell_k = (\bm v + \bm w \times \bm r_k) \cdot f_k \bm u_k.
\end{equation}
Therefore
\begin{equation}
\dot \ell_k = \bm u_k \cdot \bm v + \bm r_k \times \bm u_k \cdot \bm w .
\end{equation}
The result follows by equation~\eqref{dot l Lambda phi}.
\end{proof}

For complex situations, where the cables might pass through pulleys, the simplicity of equation~\eqref{L ell_k} can be a bit surprising.  This might best be intuitively understood by seeing that the involute of the curve describing the shape of the pulley is the curve traced by the end effector attachment point when the cable length of that actuator is kept fixed, and that the evolute is the opposite process.

\begin{proof}[Proof of Equation~\eqref{L u_k - fixed}]
Note that
\begin{equation}
\mathcal L_\psi \bm u_k
= - \frac1{|\bm r_k - \bm s_k|}\mathsf P_{\bm u_k} \mathcal L_\psi \bm s_k.
\end{equation}
Now apply equation~\eqref{lie deriv vector 1}.
\end{proof}

\bigskip

Next, we justify the Newton-Raphson Methods for forward kinematics.  The method as usually states only applies to linear vector spaces, whereas we are working on the non-linear manifold of unit dual quaternions.  Thus given $\eta_k$, we need to define a map from the vector space of vector dual quaternions to unit dual quaternions close to $\eta_k$.  See, for example, \cite{huper-trumpf}.

Most papers on the Newton-Raphson Method on manifolds construct this map using the so called exp function \cite{dedieu-et-al,fernandes-et-al,ferreira-svaiter}.  So the map is
\begin{equation}
\theta \mapsto \eta_k \exp(\theta).
\end{equation}
The exp map in these papers is following the path of a geodesic on the manifold, and this is equivalent to using the equations of motion of the end effector as described in Section~\ref{dynamics}.  Another exp map is to follow a one-parameter subgroup, or equivalently, equation~\eqref{exponential}.  We do not take these approaches.  Our approach is to normalize:
\begin{equation}
\theta \mapsto \eta_k  \widehat{(1 + \theta)}.
\end{equation}
(However, this is numerically close to the second approach, as is shown by equation~\eqref{exp-normal}.  Using normalization instead of the exp map slightly reduces the complexity of the calculations as transcendental trigonometric functions are not required.  But if one wants to use the exp map, simply replace $\widehat{(1+\theta)}$ by $\exp(\theta)$ throughout.)

The main substantive difference between these methods, and the standard method that is used on linear vector spaces,  is that the map from the linear space of vector dual quaternions to the manifold of unit dual quaternions changes with each iteration, since the map depends upon $\eta_k$.  But the theory that the Newton-Raphson Method converges with quadratic order is based upon examination of each iteration separately, so this shouldn't pose a great issue.

\begin{proof}[Justification of Equation~\eqref{exact constrained Newton-Raphson}]
If $\mathsf F:\mathbb R^6 \to \mathbb R^6$ is a function for which we wish to solve for $\mathsf F(\bm x) = 0$, the method is to iterate
\begin{equation}
\bm x - \left[\frac{\partial \mathsf F}{\partial \bm x}\right]^{-1} F(\bm x) .
\end{equation}
Our approach, then, is to consider the map
\begin{equation}
\theta \mapsto \mathsf F(\theta) := \mathsf L(\eta_k \widehat{(1 + \theta)}) - \bm\ell.
\end{equation}
The prior guess is then $\theta_k = 0$.  We have that
\begin{equation}
\frac{\partial \mathsf F}{\partial \theta} \Big|_{\theta = 0} = \mathsf \Lambda(\eta_k) .
\end{equation}
The result follows.
\end{proof}

\begin{proof}[Justification of Equation~\eqref{over constrained Newton-Raphson}]
We seek to find the pose $\eta$ so that $\mathsf L(\eta)$ is close as possible to $\bm\ell$.  This is performed by minimizing the \emph{loss function}
\begin{equation}
b(\eta) = \tfrac12 {|\mathsf L(\eta) - \bm\ell|}^2 .
\end{equation}
The standard Newton-Raphson Method for optimizing the real valued quantity $F(\bm x)$, where $\bm x$ is an element of a vector space, is to iterate
\begin{equation}
\bm x - \left[\frac{\partial^2 F}{\partial \bm x^2}\right]^{-1} \frac{\partial F}{\partial \bm x} .
\end{equation}
In our case $\bm x$ is $\theta$,
\begin{equation}
F(\theta) = b(\eta_k\widehat{(1+\theta)}),
\end{equation}
and our previous iterate is $\theta_k = 0$.  The Jacobian is
\begin{equation}
\label{jacobian}
\delta_k = \left.\frac{\partial F(\theta)}{\partial\theta} \right|_{\theta = 0} = (\liederiv_i b(\eta_k))_{1 \le i \le 6}
= \left(\sum_{m=1}^n \mathsf \Lambda_{m,i} (\mathsf L_m(\eta_k) - \ell_m)\right)_{1 \le i \le 6},
\end{equation}
and the Hessian is $\mathsf H_k$, which by equation~\eqref{hessian} is
\begin{equation}
\mathsf H_k
= \left[\tfrac12 (\liederiv_i \liederiv_j b(\eta_k) + \liederiv_j \liederiv_i b(\eta_k)) \right]_{1\le i,j\le 6},
\end{equation}
noting that
\begin{equation}
\liederiv_i \liederiv_j b(\eta_k) = \sum_{m=1}^n \mathsf\Lambda_{m,i} \mathsf\Lambda_{m,j} + \mathcal L_i \mathsf \Lambda_{m,j} (\mathsf L_m(\eta_k) - \ell_m) .
\end{equation}
\end{proof}

\bigskip

Now we work on proving Theorem~\ref{equation of motion}.  We use the Euler-Lagrange equations.  (However, one could also use standard formulas for rotating bodies, and Newtonian physics, to obtain the same result.)

Define the \emph{cross product} of two vector dual quaternions $\alpha = \bm a + \epsilon \bm b$ and $\beta = \bm c + \epsilon \bm d$ by
\begin{equation}
\alpha \times \beta = \tfrac12(\alpha \beta - \beta \alpha) = \bm a \times \bm c + \epsilon (\bm a \times \bm d + \bm b \times \bm c) .
\end{equation}
Define the \emph{adjoint products} of vector dual quaternions by
\begin{align}
\label{ltimes}
\alpha \ltimes \beta &= \bm c \times \bm a + \bm d \times \bm b + \epsilon (\bm c \times \bm b) ,\\
\alpha \rtimes \beta &= - \beta \ltimes \alpha ,
\end{align}
which can also be defined by the property that for all vector dual quaternions $\alpha$, $\beta$, and $\gamma$ we have
\begin{equation}
(\alpha \times \beta) \cdot \gamma = \alpha \cdot (\gamma \ltimes \beta) = \beta \cdot (\alpha \rtimes \gamma ).
\end{equation}
(Note that in the Lie algebra literature, the map $\beta \mapsto \alpha \times \beta$ is often denoted $\text{ad}_\alpha$.  Thus the map $\gamma \mapsto \alpha \rtimes \gamma$ is the formal dual of $\text{ad}_\alpha$.)


\begin{Theorem}
\label{euler-lagrange}
If the kinetic energy $e$ satisfies equation~\eqref{ke}, and $v = v(\eta)$ denotes the potential energy, then the equation of motion is
\begin{equation}
\label{tau M gamma}
\tau = \mu_1 + \mu_2 + \mathsf M \alpha,
\end{equation}
where
\begin{equation}
\mu_1 = \liederiv_\varphi \mathsf M \varphi - \tfrac12 \liederiv(\varphi \cdot \mathsf M \varphi) + \liederiv v, 
\end{equation}
and for any constant vector dual quaternion $\psi$ we have
\begin{equation}
\psi \cdot \mu_2 = 2 \varphi \cdot \mathsf M (\psi \times \varphi) = 2 \mathsf M \varphi \cdot (\psi \times \varphi),
\end{equation}
that is,
\begin{equation}
\label{formula mu_2}
\mu_2 = 2 (\mathsf M \varphi) \ltimes \varphi .
\end{equation}
\end{Theorem}

\begin{proof}  In preparation to apply the Euler-Lagrange Equation, given $\eta_0 \in \setunitdualquat$, we define a map from an open neighborhood of the origin in $\mathbb R^6$ to an open neighborhood of $\eta_0$ in $\setunitdualquat$
\begin{equation}
\label{eta(theta)}
\begin{aligned}
\theta &\mapsto \eta(\theta) \\
&= \eta_0 \widehat{(1+\theta)} \\
&= \eta_0 (1 + \theta + \tfrac12 \theta^2) + O(\theta^3),
\end{aligned}
\end{equation}
where in the last inequality we used equation~\eqref{approx normalize}.  Then we have
\begin{equation}
\label{dot theta}
\begin{aligned}
\varphi &= \eta^{-1} \dot\eta \\
&= \widehat{(1+\theta)}^{-1} (\dot \theta + \tfrac12(\dot\theta\theta + \theta\dot\theta)) + O(\theta^2) \\
&= \dot\theta + \tfrac12(\dot\theta\theta - \theta\dot\theta) + O(\theta^2) ,
\end{aligned}
\end{equation}
and
\begin{equation}
\label{ddot theta}
\alpha = \ddot\theta + O(\theta) .
\end{equation}
The Lagrangian of the parallel robot is
\begin{equation}
l = e - v = \tfrac12 \varphi \cdot \mathsf M \varphi - v.
\end{equation}
We use the local coordinate system $\theta$ given by equation~\eqref{eta(theta)}.  The Euler-Lagrange Equation \cite{arnold,goldstein-et-al} tells us
\begin{equation}
\label{e-l}
\frac d{dt}\left(\frac{\partial l}{\partial \dot\theta}\right) - \frac{\partial l}{\partial \theta} = \tau.
\end{equation}
We suppose that $\eta_0 = \eta(t_0)$, and $\theta(t_0) = 0$, and from now on in this proof, all equations are stated assuming the condition $t = t_0$.  Thus we only prove our results when $t = t_0$.  But since $t_0$ is arbitrary, this is not a limitation.  However, it is important that derivatives are calculated before setting $t = t_0$.  In particular, this means that for any function $f$ of $\eta$ that
\begin{equation}
\frac{\partial f}{\partial \theta} = \liederiv f.
\end{equation}

We have
\begin{equation}
\begin{aligned}
\frac d{dt}\left(\frac{\partial e}{\partial \dot\theta}\right)
&= \mathsf M\alpha + \dot{\mathsf M} \varphi \\
&= \mathsf M\alpha + \left(\dot\theta \cdot \frac{\partial \mathsf M}{\partial\theta}\right) \dot\theta \\
& = \mathsf M\alpha + \left(\varphi \cdot \frac{\partial \mathsf M}{\partial\theta}\right) \varphi,
\end{aligned}
\end{equation}
and
\begin{equation}
\label{de/dtheta}
\begin{aligned}
\frac{\partial e}{\partial \theta}
&= \dot \theta \cdot \mathsf M \frac{\partial}{\partial\theta}(\dot\theta \theta - \theta \dot\theta) + \tfrac12 \dot\theta \cdot \left(\frac{\partial \mathsf M}{\partial\theta}\right) \dot\theta \\
&= \varphi \cdot \mathsf M \frac{\partial}{\partial\theta}(\varphi \theta - \theta \varphi) + \tfrac12 \varphi \cdot \left(\frac{\partial \mathsf M}{\partial\theta}\right) \varphi .
\end{aligned}
\end{equation}
Note that if $f$ is any linear function whose domain is the vector dual quaternions, then
\begin{equation}
\psi \cdot \left(\frac{\partial}{\partial \theta} f(\theta)\right) = f(\psi) .
\end{equation}
Thus taking the dot product of equation~\eqref{de/dtheta} with any constant vector dual quaternion $\psi$, we obtain the result.
\end{proof}



\begin{Lemma}
\label{lambda pot energy}
If $v$ is the potential energy of the end effector in a constant gravity field, whose value with respect to the moving frame is $\bm g$, then
\begin{equation}
\liederiv v = - 2 m_0 (\bm r_0 \times \bm g + \epsilon \bm g) .
\end{equation}
\end{Lemma}

\begin{proof}  Let $\bm r$ be a constant point expressed with respect to the moving frame.  We have
\begin{equation}
v = m_0 \bm g \cdot (\bm r - \bm r_0) .
\end{equation}
If $\theta = \tfrac12\bm a + \tfrac12 \epsilon \bm b$ is a vector dual quaternion, then by equation~\eqref{lie deriv vector 2}, it follows that
\begin{equation}
\begin{aligned}
\mathcal L_\theta v &= - m_0 ((\bm a \times \bm g) \cdot (\bm r - \bm r_0) + \bm g \cdot (\bm a \times \bm r + \bm b)) \\
&= -m_0 (\bm a \cdot \bm r_0 \times \bm g + \bm b \cdot \bm g ) .
\end{aligned}
\end{equation}
\end{proof}

\begin{proof}[Proof of Theorem~\ref{equation of motion}] The potential energy part is covered by Lemma~\ref{lambda pot energy}.  For the parts coming from the kinetic energy, using linearity, it is sufficient to prove it for the additive parts of $\mathsf M$.  The part not involving $m_0$ is proved using Theorem~\ref{euler-lagrange}, various vector identities, and remembering equation~\eqref{r star to r cross}.

So we only need to prove the kinetic energy portion in the case $\mathsf M = \mathsf \Lambda^T \mathsf M_0 \mathsf \Lambda$.  The easiest way to show this is to simply differentiate $\mathsf M_0 \dot \ell = \mathsf M_0 \mathsf\Lambda\varphi$ with respect to time.  To do it directly from the formulas is more complicated, as we now show.  For any constant vector dual quaternion $\psi$, we have
\begin{equation}
\begin{aligned}
\psi &\cdot (\liederiv_\varphi(\mathsf M \varphi) - \tfrac12 \liederiv(\varphi \cdot \mathsf M \varphi)) \\
&= (\liederiv_\varphi(\psi \cdot \mathsf M \varphi) - \tfrac12 \liederiv_\psi(\varphi \cdot \mathsf M \varphi)) \\
&= \liederiv_\varphi ((\mathsf \Lambda \psi) \cdot \mathsf M_0 (\mathsf \Lambda \varphi))
- \tfrac12 \liederiv_\psi  ((\mathsf \Lambda \varphi) \cdot \mathsf M_0(\mathsf \Lambda \varphi)) \\
&= \liederiv_\varphi \liederiv_\varphi \mathsf L \cdot \mathsf M_0 \liederiv_\psi \mathsf L
+ \liederiv_\varphi \liederiv_\psi \mathsf L \cdot \mathsf M_0\liederiv_\varphi \mathsf L - \liederiv_\psi \liederiv_\varphi \mathsf L \cdot \mathsf M_0\liederiv_\varphi \mathsf L \\
&= \liederiv_\varphi (\mathsf\Lambda \varphi) \cdot \mathsf M_0\mathsf\Lambda \psi
+ \liederiv_{(\varphi \psi - \psi \varphi)} \mathsf L \cdot \mathsf M_0\mathsf\Lambda \varphi \\
&= \psi \cdot \mathsf\Lambda^T \mathsf M_0 \liederiv_\varphi (\mathsf \Lambda \varphi)
+ \mathsf \Lambda (\varphi\psi - \psi\varphi) \cdot \mathsf M_0 \mathsf \Lambda \varphi
\\
&= \psi \cdot \mathsf\Lambda^T \mathsf M_0\liederiv_\varphi (\mathsf \Lambda \varphi)
+ (\varphi\psi - \psi\varphi) \cdot \mathsf M_0 \varphi,
\end{aligned}
\end{equation}
where we used equation~\eqref{lie bracket diff = diff lie bracket}.  Then it is simply a matter of collecting terms.
\end{proof}

\section{Conclusion}

This paper has given a comprehensive and consistent description of how to use dual quaternions to represent poses, rigid motions, twists, and wrenches.  We have introduced the notion of the Lie derivative for dual quaternions.  We have shown how these formula are helpful for first producing Newton-Raphson methods for solving the forward kinematics problem for parallel robots, and secondly for a self contained derivation for the dynamic equations of motion of the end effector that includes the inertia of the actuators.

Finally, in equation~\eqref{approx normalize}, we give an approximation of the normalization of a vector dual quaternion perturbation of the identity, which shows that it is equal up to the second order to the exponential of the vector dual quaternion.  This equation was essential for calculating the Hessian in the forwards kinematics algorithms.  We feel that this formula will be of independent interest to other researchers in the field of dual quaternions.

\end{document}